\documentclass[letterpaper]{article}
\usepackage{url}
\usepackage{amsmath}
\usepackage{amsfonts, latexsym, amssymb, amsthm} 
\usepackage{color}
\usepackage{graphicx}
\usepackage[hang,flushmargin]{footmisc}

\usepackage{natbib}
\usepackage{algorithm}
\usepackage{algorithmic}
\usepackage{subfigure}

\DeclareGraphicsExtensions{.pdf,.png,.gif,.jpg}

\theoremstyle{plain}
\newtheorem{theorem}{Theorem}
\newtheorem{lemma}[theorem]{Lemma}

\newtheorem{corollary}[theorem]{Corollary}
\theoremstyle{definition}
\newtheorem{definition}{Definition}
\newtheorem{assumption}{Assumption}

\theoremstyle{remark}
\newtheorem{remark}{Remark}

\newcommand{\iid}{i.\@i.\@d.\@ }

\newcommand{\ie}{\emph{i.e.}, }
\newcommand{\ud}{\mathrm{d}}

\providecommand{\abs}[1]{\left\lvert#1\right\rvert}

\providecommand{\eps}{\varepsilon}
\newcommand{\E}{{\mathbb E}}

\renewcommand{\P}{{\mathbb P}}

\newcommand{\neww}[1]{}

\DeclareMathOperator{\VAR}{V{@}R}
\DeclareMathOperator{\AVAR}{AV{@}R}
\DeclareMathOperator{\var}{\mathbb V}

\renewcommand{\hat}{\widehat}
\renewcommand{\leq}{\leqslant}
\renewcommand{\geq}{\geqslant}

\title{Functional Bandits}

 \author{Long Tran-Thanh\thanks{ECS, Faculty of Physical Sciences and Engineering,
University of Southampton, Southampton, United Kingdom
(ltt08r@ecs.soton.ac.uk).} \@{} and Jia Yuan Yu\thanks{IBM Research Ireland,
     Damastown, Dublin, Ireland (jy@osore.ca). This work is supported in part by the EU FP7 project INSIGHT under grant 318225.}}

\begin{document}

\maketitle

\begin{abstract}
We introduce the functional bandit problem, where the objective
is to find an arm that optimises a known functional of the unknown arm-reward distributions.
These problems arise in many settings such as maximum entropy
methods in natural language processing, and risk-averse decision-making,
but current best-arm identification techniques 
fail in these domains.
We propose a new approach, that combines functional estimation and arm elimination, to tackle this problem. 
This method achieves provably efficient performance guarantees.
In addition, we illustrate this method on a number of important functionals in risk management and information theory, and refine our generic theoretical results in those cases.    
\end{abstract}



\section{Introduction}

\noindent
The stochastic multi-armed bandit (MAB) model consists of a slot machine with $K$ arms (or actions), each of which delivers rewards that are independently and randomly drawn from an unknown distribution when pulled. 
In the optimal-arm identification problem, the aim is to find an arm with the highest expected reward value. 
To do so, we can pull the arms and learn (i.e., estimate) their mean rewards. 
That is, our goal is to distribute a finite budget of $T$ pulls among the arms, such that at the end of the process, we can identify the optimal arm as accurately as possible.
This stochastic optimisation problem models many practical applications, ranging from keyword bidding strategy optimisation in sponsored search~\citep{Amin_all}, to identifying the best medicines in medical trials~\citep{Robbins1952}, and efficient transmission channel detection in wireless communication networks~\citep{Avner2012}.  

Although this MAB optimisation model is a well-studied in the online learning community, the focus is on finding the arm with the highest expected reward value~\citep{MaronAndMoore1993,MnihEtAl2008,ABM2010,KKS13}.
However, in many applications, we are rather interested in other statistics of the arms, which can be represented as functionals of the corresponding distribution of the arm-reward values.
For example, in finance and robust optimisation, notions of monetary value of risk are typically captured by risk measure functionals.
In addition, telecommunication and natural language processing metrics are usually captured by information theoretic functionals (e.g., entropy and divergence). 
Existing optimal-arm identification techniques cannot be applied to other functionals in a straightforward way, as they exploit the fact that the expected value can be estimated in a consistent way, without any bias. 
This property, however, does not always hold in the nonasymptotic regime for other functionals such as entropy, divergence, or some risk measures (e.g., value-at-risk, average value-at-risk). 


Against this background, we introduce a general framework, called \emph{functional bandit optimisation}, where optimal-arm identification means finding an arm with the optimal corresponding functional value.  
To do so, we first propose Batch Elimination, an efficient arm elimination algorithm (i.e., optimisation method), that is suitable for identifying the best arm with a small number $T$ of trials.
The proposed algorithm is a generalised version of the Successive Elimination~\citep{ABM2010} and Sequential Halving~\citep{KKS13} methods.

We provide generic theoretical performance guarantees for the algorithm.
We refine our results in a number of scenarios with some specific and important functionals. 
In particular, we focus on those with applicability to risk management and information theory.
Given this, for risk management, we investigate the mean-variance, value-at-risk, and average-value-at-risk functionals.
Furthermore, we also study Shannon entropy functional, a widely used information theoretic metrics.

The following are our main contributions.
We start with introducing the problem of optimal-arm identification with functional bandits, a generalised framework for the existing best arm identification model.
We then propose Batch Elimination, which can be regarded as a generalised version of many existing optimisation methods.  We also provide rigorous theoretical performance analysis for the algorithm.
In the following sections, we refine our results to a number of practical risk management and information theoretic functionals.  
%
%
The last section concludes with questions and discussions.



\section{Related work}


\noindent
Multiarmed bandit problems have been studied in a variety of settings, including 
the Markovian (rested and restless), stochastic and 
adversarial settings.  For surveys on bandit problems, we refer the
reader to \citep{Gittins,CesLug06}.
Two types of results are found in the literature: results on the average regret (i.e., in a regret-minimization setting \citep{LaiRob85}) and results on the sample complexity (i.e., in a pure-exploration setting \citep{EveManMan02,DGW2002,BMS11}).
Our work is of the second type. It is related to work on sample complexity of bandit arm-selection \citep{EveManMan02,KTAS12}, which is also known as pure exploration or best-arm identification \citep{ABM10,GGLB12}.

In the Markovian setting, \citep{DPR07,ChaLarPal07} consider a one-armed bandit problem in the
setting of Gittins indices and model risk with concave utility
functions.
In the stochastic setting, the notion of risk has been limited to
empirical variance \citep{AMS08,SLM12} and risk measures \citep{YN13}.
In \citep{AMS08,SLM12}, the functionals assign real values to the decision-maker's \emph{policies} (\ie confidence-bound algorithms) and guarantees are given for the \emph{regret} in retrospect.
As in \citep{YN13}, our functionals assign a real value to random variables, \ie \emph{rewards} of individual arms. This is more in line with the risk notions of the finance and optimization literature.



\section{The Functional Bandit Model}

\noindent
Our bandit model consists of $K$ arms.
By pulling a particular arm $i \in[1,\dots,K]$, we receive a reward $X_{i}$ drawn from an unknown stationary distribution $F_{i}$ (i.e., repeatedly pulling the same arm results in generating a sequence of i.i.d. random variables).
Suppose $G(.)$ is a functional of $F_{i}$, and we denote the its value for arm $i$ as $G_{i} = G(F_{i})$.
Our goal is to identify the arm with the best (e.g., highest, or lowest) functional value. 
For the sake of simplicity, we assume that the highest functional is the best. 
That is, we aim to find $i^* = \arg\max_{i}{G_{i}}$.
However, as $F_{i}$ are initially unknown, we aim to achieve this goal by using an arm pulling policy which works as follows.
For $t = 1,\dots, T$ finite number of time steps, at each $t$, the policy chooses an arm $i(t)$ to pull, and observes the received reward. 
At the end of $T$, the policy chooses an arm, denoted as the random variable $i^+(T)$, which
it believes has the best functional value. 
The regret of an arm-$i$ is defined as $\gamma_{i} \triangleq  G_{i^*} - G_{i}$.
The expected regret of the policy is:
\begin{equation}
\nonumber
r(T) = \mathbb{E}[G_{i^*} - G_{i^+(T)}] = \mathbb{E}[\gamma_{i^+(T)}].
\end{equation}
In addition, let 
$$e_r(T) = \P(i^+(T) \neq i^*)$$ 
denote the probability 
that we recommend a suboptimal arm after $T$ samples.
Our goal is then to find a policy that achieves minimal regret and the recommendation error.
In what follows, we describe a generic algorithm that is designed to efficiently identify the best arm.


\subsection{The Batch Elimination Algorithm}
\label{subsection:alg}

\begin{algorithm}[t!]
\begin{algorithmic}[1]
\STATE \textbf{Inputs}: $T$, $\{x_1,\dots, x_L\}$.
\STATE Initialize: $S_1 = \{1,\dots, K\}$.
\FOR{$l=1$ \TO $L$}
\STATE pull each remaining arm $i \in S_m$ for $\lfloor T/H \rfloor$ times,
\STATE use estimator $\hat{G}_{i}$ to estimate $G_{i}$,
\STATE eliminate the weakest $x_l$ arms from $S_l$ to obtain $S_{l+1}$.
\ENDFOR
\STATE \textbf{Output}: $i^+(T)$, which is the sole arm in $S_L$.
\end{algorithmic}
\caption{The Batch Elimination Algorithm}
\label{alg:elimination}
\end{algorithm}

\noindent
We now turn to the description of Batch Elimination, our arm elimination algorithm. 
Its pseudo code is depicted in Algorithm~\ref{alg:elimination}. 
For a given integer $L > 0$, let $\{x_1,\dots, x_L\}$ be a sequence of
$L$ non-negative integers such that $\sum_{l=1}^{L}{x_l} = K -1$.


%
%
The Batch Elimination algorithm runs over $L$ rounds. Within each round
$l=1,\ldots,L$, it maintains a set $S_l$ of remaining arms, and it pulls each of
the arms within this set $\lfloor T/H\rfloor$ times, where the value of $H$ is defined later. 
It then uses the corresponding reward-samples to update the functional
estimate $\hat{G}_{i}$ of $G_{i}$ of each remaining arm. 
Here, we assume that we have access to an estimator for each functional $G_{i}$, which can be calculated from the samples drawn from $F_{i}$.
This estimator must have a property that will be defined later.
Finally, we eliminate the weakest $x_m$  arms (i.e., those with the
lowest estimates for their functionals) and proceeds to the next round.

Observe that since $\sum_{l=1}^{L}{x_l} = K -1$, the above algorithm repeats until one arm remains.
Our algorithm can be regarded as a generalised version of the
Successive Rejects~\citep{ABM2010} and Sequential
Halving~\citep{KKS13} algorithms, which are designed for identifying the arm with highest \emph{mean value}.
In fact, by setting $L = K-1$ and $x_1 = x_2 = \dots = x_L = 1$, we get the Successive Rejects method, where we only eliminate the weakest remaining arm at each round.
On the other hand, by setting $L = \lceil \log_2{K}\rceil$, $x_1 =
\lfloor K/2 \rfloor$ and $x_l =  \lfloor (K -
\sum_{j=1}^{l-1}{x_j})/2 \rfloor$ for $1 < l \leq L$, we get the Sequential Halving algorithm, where we eliminate the weaker half of the remaining arms at each round.

It is remained to set the value of $H$. 
Recall that at each round, we pull each arm $\lfloor T/H \rfloor$ times.
Given this, we have to choose the value of $H$ such that the total number of pulls within the algorithm does not exceed $T$.
Let

\begin{equation}
H = LK - \sum_{l=1}^{L}{x_l(L-l)}.
\end{equation}

We show that by doing so, we can guarantee that our algorithm does not pull more than $T$ arms. 
Indeed, observe that at each round $l$, the total number of pulls is $|S_l| \lfloor T/H\rfloor$. 
Hence, the total number of pulls is
\begin{equation}
\nonumber
\sum_{l=1}^{L}{|S_l| \Big\lfloor \frac{T}{H} \Big\rfloor} \leq \sum_{l=1}^{L}{  \Big(K - \sum_{j=1}^{l-1}{x_j}\Big)\frac{T}{H}} = H\frac{T}{H} = T.
\end{equation}

\section{Performance Analysis}
\label{section:analysis}

\noindent
Given the description of our arm elimination algorithm, we now turn to the investigate its performance. 
To do so, we first define when an estimator is considered to be sufficiently good.
Let $\hat{G}(n)$ denote the estimate value of $G(F)$ by using estimator $\hat{G}$ on $n$ i.i.d. samples from $F$. 
We say that:
\begin{definition}[Q-efficient estimator]
An estimator $\hat{G}$ of functional $G$ is \emph{Q-efficient} if there
exists a function $Q(n,x): \mathbb{R} \times \mathbb{R} \rightarrow
\mathbb{R}$, monotone decreasing in both $n$ and $x$, such that for
every number of sample $n > 0$ and real number $x > 0$, we have
\begin{eqnarray}
\nonumber
\P(\hat{G}(n) - G \geq x) \leq Q(n,x) \\
\nonumber
\mbox{and }\P(G - \hat{G}(n) \geq x) \leq Q(n,x)
\end{eqnarray}
\end{definition}

\noindent
A Q-efficient estimator has the property that the probability that
the one-sided estimation error exceeds $x$ is bounded by $Q(n,x)$.
Intuitively, this property guarantees that the estimate $\hat{G}$ is not too far away from the true value.
We demonstrate in the subsequent sections that many practical functionals have Q-efficient estimators.

Next, we turn to the analysis of the Batch Elimination algorithm.
First, we bound the probability of obtaining a wrong
ordering of two arms from Q-efficient estimators.

\begin{lemma}[Wrong order]
\label{lemma:incorrect_prob}
Consider a fixed round $l$.
Suppose that $\hat G_{i}$ is a Q-efficient estimator for each $i \in S_l$
and that we already have $n$ samples from each remaining arm.
In addition, suppose that $i^*$ has not been eliminated before this round.
Then, for all $i \in S_l$ and $i \neq i^*$, we have:
\begin{align*}
  \P( \hat G_{i}(n) \geq \hat G_{i^*}(n) ) \leq 2Q \Big(n,\frac{\gamma_i}{2} \Big)
\end{align*}
\end{lemma}

\noindent
Next, we bound the probability of eliminating the optimal arm $i^*$.
\begin{lemma}[Eliminating optimal arm]
\label{lemma:error_one_round_generic}
Suppose that the assumptions of Lemma~\ref{lemma:incorrect_prob} hold.
Let $d \triangleq \min_{i \neq i^*}{\gamma_{i}}$.
Suppose that the best arm $i^*$ is not eliminated until round $l$. 
The probability that the optimal arm is eliminated at the end of
round $l$ is bounded as follows:
\begin{equation}
\nonumber
\P(i^* \not\in S_{l+1} \mid i^* \in S_l) \leq 2(|S_l| - x_l)Q\Big(n,\frac{d}{2}\Big).
\end{equation}
\end{lemma}

\noindent
%
Our main result can be stated as follows. 
\begin{theorem}[Main result]
\label{thm:generic_error_probability}
Suppose that the assumptions of Lemma~\ref{lemma:error_one_round_generic} hold.
We have the following upper bound for the recommendation error:
\begin{align*}
e_r(T) \leq 2\Big(H - K + 1\Big)Q\Big(\frac{T- H}{H},\frac{d}{2}\Big).
\end{align*}
\end{theorem}

\noindent
Alternatively, the following corollary provides a minimal number of pulls that can guarantee a high probabilistic success.
\begin{corollary}[Sample complexity]
\label{cor:min_pulls}
Let $0 < \delta < 1$ and let $Q^{-1}_{d/2}(.)$ denote the inverse function of 
$Q_{d/2}(n) = Q\Big(n,\frac{d}{2}\Big)$.
If $Q$ is strictly monotone in $n$, then if
$$T \geq HQ^{-1}_{d/2}\Big(\frac{\delta}{2H - 2K + 2} \Big) + H,$$
our success probability (i.e., correctly recommending $i^*$) is at least $(1 - \delta)$.
\end{corollary}

\begin{proof}[Proof sketch of Theorem~\ref{thm:generic_error_probability}]
Let $e_p(l)$ denote the probability that the best arm is not eliminated until round $l$ but will be dropped out at this round.
Given this, we have:
\begin{eqnarray}
\nonumber
e_r(T) &\leq& \sum_{j=1}^{L}{e_p(j)} \leq \sum_{j=1}^{L}{2(|S_j| - x_j)Q\Big(l\frac{T-H}{H},\frac{d}{2}\Big)} \\
\nonumber
&\leq&  2Q\Big(\frac{T-H}{H},\frac{d}{2}\Big)\sum_{j=1}^{L}{\Big(K - \sum_{r=1}^{l-1}{x_r} - x_j \Big)} \\
\nonumber
&\leq& 2Q\Big(\frac{T-H}{H},\frac{d}{2}\Big)\Big( H - K + 1\Big) 
\end{eqnarray}
which concludes proof.
\end{proof}
\begin{proof}[Proof of Corollary~\ref{cor:min_pulls}]
Note that since $Q_{d/2}$ is strictly monotone decreasing in $n$, $Q^{-1}_{d/2}$ always exists.
In addition, we have:
\begin{align*}
e_r(T) &\leq 2\Big(H - K + 1\Big)Q\Big(\frac{T- H}{H},\frac{d}{2}\Big) \\
& \leq 2\Big(H - K + 1\Big)Q\Big(\frac{T_0- H}{H},\frac{d}{2}\Big)
\end{align*}
where $T_0 = HQ^{-1}_{d/2}\Big(\frac{\delta}{2H - 2K + 2} \Big) + H$. 
Simplifying the last term we obtain $e_r(T) \leq \delta$.
\end{proof}

\noindent
Ou main result, Theorem~\ref{thm:generic_error_probability}, also implies the following corollaries.
\begin{corollary}[Regret bound]
\label{corollary:regret_bound}
We have the following upper bound for the expected regret
$$r(T) \leq  2\gamma_{\mathrm{max}} \Big(H - K +1 \Big)Q\Big(\frac{T- H}{H},\frac{d}{2}\Big)$$
where $\gamma_{\mathrm{max}} = \max_{j \neq i^*}{\gamma_j}$.
\end{corollary}
\begin{corollary}[PAC regret bound]
\label{corollary:regret_bound_PAC}
Let the algorithm output after $T$ time steps be $i^+(T)$, then for any $0 < \delta < 1$, we have
\small
 \begin{align*}
G_{i^*} - G_{i^+(T)} &\leq \frac{2\gamma_{\mathrm{max}}}{\delta} \Big(H -K + 1\Big)Q\Big(\frac{T- H}{H},\frac{d}{2}\Big) 
\end{align*}
\normalsize
with at least $(1-\delta)$ probability. 
\end{corollary}

\noindent
These results provide the upper bound for the algorithm's regret and probably approximately correct (PAC) regret, respectively. 
The latter is a regret that holds with high probability, in contrast to the former, which always holds.

\begin{proof}[Proof sketch of Corollary~\ref{corollary:regret_bound}]
By definition, we have 
\begin{eqnarray}
\nonumber
r(T) &=& \sum_{j \neq i^*}{\gamma_j \P\Big(I^+ = j\Big)} \leq \gamma_{\mathrm{max}}\sum_{j \neq i^*}{\P\Big(I^+ = j\Big)} \\
\nonumber
&\leq&  \gamma_{\mathrm{max}} e_r(T) \leq 2\gamma_{\mathrm{max}}\Big(H -K + 1 \Big)Q\Big(\frac{T- H}{H},\frac{d}{2}\Big)
\end{eqnarray}
\end{proof}
\begin{proof}[Proof sketch of Corollary~\ref{corollary:regret_bound_PAC}]
Let $\varepsilon > 0$ be an arbitrary number.
From Markov's inequality, we have
\begin{eqnarray}
\nonumber
 \P\Big( G_{i^*} - G_{i^+(T)} \leq \varepsilon \Big) \geq 1- \frac{\mathbb{E}[G_{i^*} - G_{i^+(T)}]}{\varepsilon} \geq 1- \frac{r(T)}{\varepsilon}
\end{eqnarray}
This also holds for $\varepsilon = r(T)/\delta$ where $0 <
\delta < 1$. An application of Corollary~\ref{corollary:regret_bound} concludes the proof.
\end{proof}

As our results hold for a large class of functionals, they are not comparable with the existing optimal arm identification results in general, as the latter are only designed for the expected value functional.
However, by applying our results to the case of expected values, we can still compare them against the state of the art.
In what follows, we will make a comparison between our performance guarantees and the existing bounds of~\cite{ABM2010}, and~\cite{KKS13}, respectively.

In particular, we can show that the expected value functional is Q-efficient with $Q(n,x) = \exp{\Big\{ -n x^2\Big\} }$. 
Indeed, this can be proved by using Hoeffding's inequality.
Now, by setting the values of $L$ and $x_l$ to be  $L = \lceil \log_2{K}\rceil$, $x_1 = \Big \lfloor K/2 \Big\rfloor$ and $x_l =  \Big \lfloor (K - \sum_{j=1}^{l-1}{x_j})/2 \Big\rfloor$ for $1 < l \leq L$, we can prove that our algorithm has
$$(K + \log_2{K})\exp{\Big\{ -\frac{(T-2K)d^2}{8K}\Big\}}$$
upper bound for the recommendation error probability $e_r(T)$ when applying our results to the expected value case in a straightforward way (i.e., without using further technical refinements).
Note that the upper bound of the same error probability is 
$$\frac{K(K-1)}{2} \exp{ \Big\{-O\Big(\frac{(T-K)d^2}{K\log_2{K}} \Big) \Big\} }$$
in~\citep{ABM2010}, and 
$$3\log_2{K} \exp{ \Big\{-O\Big(\frac{Td^2}{8K\log_2{K}} \Big) \Big\} }$$
in~\citep{KKS13}, respectively\footnote{In these bounds, the expression in the exponential term is in fact $O(-T/(H_2\log_2{K}) )$ with $H_2 = \max_{i \neq i^*}{i/\gamma_i^2}$. However, we can roughly estimate $\frac{1}{H_2}$ with $O(d^2/K)$.}. 
Roughly speaking, our error bound has better constant coefficient (i.e., $K + \log_2{K}$), compared to that of~\cite{ABM2010}, but weaker than that of in ~\citep{KKS13}.
On the other hand, it outperforms both existing bounds from the aspect of the exponential term.
In fact, the exponential term of our bound is $\exp \{ -O(\frac{T}{K})\}$, while the others have $\exp \{ -O(\frac{T}{K\log_2{K}})\}$.
This implies that for sufficiently large values of $K$, our regret bound is more efficient than the existing ones. 
Similar results can be obtained for the comparison of regret and PAC regret bounds.
It is worth to mention that by using some elementary algebra, we can improve the error bound of Batch Elimination for expected values to 
$$C\log_2{K} \exp{ \Big\{-O\Big(\frac{Td^2}{K} \Big) \Big\} }$$ 
for some constant $C >0$.
However, due to space limitations, we omit the details.

Recall that the abovementioned analyses rely on the assumption that the investigated functional has a Q-efficient estimator, which might not always hold.
However, in the next sections, we will demonstrate that many practical functionals have this property (i.e., there is a Q-efficient estimator for such functionals).
In particular, we derive specific $Q$ functions for each of these functionals, and we refine the regret bound results, described in Theorem~\ref{thm:generic_error_probability} and Corollaries~\ref{corollary:regret_bound} and~\ref{corollary:regret_bound_PAC}, tailoring to each corresponding functional.



\section{Risk Functionals}

In this section, we consider three specific instances of functionals widely
used as risk measures in decision-making. They are the mean-variance
risk, the
value-at-risk, and the average value-at-risk.

\subsection{Mean-Variance}

In this section, we consider the mean-variance objective functional
\begin{align*}
  G^{M,\lambda}_i = -\mu(i) +\lambda \sigma^2(i),
\end{align*}
where $\mu(i)$ and $\sigma^2(i)$ denote the mean and variance of arm
$i$.  The mean-variance risk measure has been used in risk-averse problem
formulations in a variety of applications in finance and reinforcement learning \cite{Markowitz52,ManTsi11,SLM12}.

Let $\lambda$ be given and fixed. We assume that we are given $N$
samples $X^i_1,\ldots,X^i_N$ for every arm $i$. We employ the
following unbiased estimate
\begin{align*}
  \hat G^{M,\lambda}_i = -\hat \mu(i) +\lambda \hat \sigma^2(i),
\end{align*}
where $\hat \mu(i)$ and $\hat \sigma^2(i)$ denote the sample-mean and
the unbiased variance estimator
$\hat \sigma^2(i) = \frac{1}{N-1} \sum_{k=1}^N (X^i_k - \hat \mu(i))^2$.


\begin{theorem}[MV PAC bound]\label{thm:mv}
  Suppose that there exist $A,B$ such that $\P(X^i_t \in [A,B]) = 1$
  for all $i$.  
Let $i^*$ denote the best arm with respect to the functional
$G^{M,\lambda}$.
Consider a fixed round $l$ and $N$ samples for each remaining arm.
Then, for all $i=1,2\ldots$, such that $i \neq i^*$ and is still not eliminated at round $l$, we have the following bound:
\begin{eqnarray}
\nonumber
  \P( \hat G^{M,\lambda}_{i}(N) \geq \hat G^{M,\lambda}_{i^*}(N) ) &\leq& 2 \exp\left(- \frac{N \gamma_i^2}{8(B-A)^2}\right) \\
\nonumber  
  &+& 2 \exp\left(- \frac{N (\frac{N-1}{N}\gamma_i/\lambda)^2}{8 (B-A)^4}\right).
\end{eqnarray}
\end{theorem}

\noindent
A straightforward application of this theorem to Lemma~\ref{lemma:error_one_round_generic} and Theorem~\ref{thm:generic_error_probability} implies the following:
\begin{corollary}
For the mean-variance functional case, our algorithm has the following upper bound for the recommendation error $e_r(T)$: 
\begin{eqnarray*}
e_r(T) &\leq& 2\Big(H - K +1 \Big)\exp\left(- \frac{(T-H)\gamma_i^2}{8H(B-A)^2}\right) \\
 &+& 2\Big(H - K +1 \Big)\exp\left(- \frac{(T-2H)^2 (\gamma_i/\lambda)^2}{8 (T-H) (B-A)^4}\right)
\end{eqnarray*}  
\end{corollary}

\noindent
We can also derive regret and PAC regret bounds for this case, similarly to Corollaries~\ref{corollary:regret_bound} and~\ref{corollary:regret_bound_PAC}.
However, due to the space limitations, we leave this to the reader.

\subsection{Value-at-Risk}



Let $\lambda$ be given and fixed.  In this section, we consider the value-at-risk, for
every arm $i$:
\begin{align*}
  G^{V,\lambda}_i = \VAR_\lambda(X^i_1) = - q_i(\lambda),
\end{align*}
where $q_i$ is the right-continuous quantile function\footnote{Formally, $q_i(\lambda) = \inf\{x\in\mathbb R : F_i(x) > \lambda\}$, where $F_i$ is the distribution function of $X^i_1$.} of $X^i_1$.

Suppose that up to time $T$, each arm is sampled $N$ times.  Let
$X^i_1, \ldots, X^i_N$ denote the sequence of rewards generated by arm
$i$.  Let $X^i_{(1)} \leq \ldots \leq X^i_{(N)}$ denote a reordering
of the random variables $\{X^i_1, \ldots, X^i_N\}$, where $X^i_{(k)}$
is the $k$-th order statistic of the sequence $\{X^i_1, \ldots,
X^i_N\}$.  We consider the following $\VAR$ estimators for all $i$:
\begin{align*}
  \hat G^{V,\lambda}_i(N) \triangleq - X^i_{(\lceil \lambda N \rceil)},
\end{align*}
where $X^i_{(\lceil \lambda N \rceil)}$ is a $\lambda$-quantile
estimator\footnote{  With slight modifications, we can derive similar results with other quantile estimators, such as the Wilks estimator.}.

\begin{assumption}[Differentiable reward density]\label{as:1}
  For each arm $i$, the reward probability density functions $d_i$ are
  continuously differentiable.
\end{assumption}

The following theorem from the theory of order statistics establishes
the convergence of the quantile estimator.

\begin{theorem}\citep{DavNag03}\label{thm:bahadur}
  Suppose that Assumption~\ref{as:1} holds.  Let $d_i$ denote the
  probability density function of arm $i$'s rewards, and $d_i'$ denote the
  derivative of $d_i$.  There exist constants $C_1,C_2 \geq 0$ and
  scalars $V^i_N$ and $W^i_N$
  \begin{align*}
    \abs{ V^i_N } &\leq \abs{ \frac{\lambda(1-\lambda) d_i'(q_i(\lambda))}{2(N+2) d_i^3(q_i(\lambda))} } + C_1/N^2,\\
    W^i_N &\leq \frac{\lambda(1-\lambda)}{(N+2) d_i^2(q_i(\lambda))} +
    C_2/N^2
  \end{align*}
  such that
  \begin{align*}
    \E X^i_{(\lceil \lambda N \rceil)} &= q_i(\lambda) + V^i_N,\\
    \var X^i_{(\lceil \lambda N \rceil)} &= \E (X^i_{(\lceil \lambda N \rceil)} - \E X^i_{(\lceil \lambda N \rceil)})^2 = W^i_N.
  \end{align*}
\end{theorem}

\noindent
The following theorem uses a result from order statistics to derive
  a PAC sample complexity bound on our estimator. 

\begin{theorem}[$\VAR$ estimation error]\label{thm:var}
  Suppose that Assumption~\ref{as:1} holds.  Suppose that the number
  of samples of each arm is $N$.
  Then, for every arm $i$, we have:
  \begin{align*}
    \P\left( \abs{\hat G^{V,\lambda}_i(N) - G^{V,\lambda}_i} > \eps \right)
    \leq \frac{W^i_N}{(\eps - \abs{V^i_N})^2}.
  \end{align*}
\end{theorem}

\noindent
This result implies the following statement:
\begin{lemma}
\label{lemma:VAR_incorrect_prob}
Suppose $i^*$ has not been eliminated until round $l$. 
Let $N$ denote the total number of pulls per remaining arm.
Then, for all $i \in S_l$ and $i \neq i^*$, we have:
$$\P( \hat G^{V,\lambda}_{i}(N) \geq \hat G^{V,\lambda}_{i^*}(N) ) \leq 2\frac{W^i_N}{(\frac{d}{2} - \abs{V^i_N})^2} $$
\end{lemma}
By applying this to Lemma~\ref{lemma:error_one_round_generic} and Theorem~\ref{thm:generic_error_probability}, we can refine the upper bound for the recommendation error probability $e_r(T)$ of our algorithm as follows:
\begin{corollary}
The recommendation error for the value-at-risk functional case has the following upper bound:
\begin{eqnarray*}
e_r(T) \leq 4\Big(H - K +1 \Big)\frac{W}{(\frac{d}{2} - \abs{V})^2} 
\end{eqnarray*}  
where $W = \max_{i \neq i^*}{W^i_{\lfloor \frac{T}{H}\rfloor}}$ and $V = \max_{i \neq i^*}{V^i_{\lfloor \frac{T}{H}\rfloor}}$.
\end{corollary}

\noindent
We can also derive regret and PAC regret bounds for this case. However, due to the space limitations, we omit these steps.

\subsection{Average Value-at-Risk}

Modern approaches to risk measures advocate the use of convex risk measures, which capture the fact that diversification helps reduce risk.
In this section, we
consider only one instance of convex risk measures: the average value-at-risk.
Nonetheless, it can be shown that an important subset of convex risk
measures (\ie those continuous from above, law invariant, and
coherent) can be expressed as an integral of the $\AVAR$
(cf. \cite{Schied04}). Guarantees can be obtained for those risk
measures by using the approach of this section.

The $\AVAR$ has the following two equivalent definitions---first, as an integral of $\VAR$:
  \begin{align*}
    G^{A,\lambda}_i = \AVAR_\lambda(X^i_1) = \frac{1}{\lambda}
    \int_0^\lambda \VAR_\phi(X^i_1) \ud\phi,
  \end{align*}
and second, as a maximum over a set of distributions:
    $\rho^A_\lambda(X) = \max_{Q \in \mathcal Q_\lambda(\P)} - \E_Q X$,
  where $\mathcal Q_\lambda(\P)$ is the set of probability measures $\{Q : \frac{\ud Q}{\ud \P} \leq 1/\lambda\}$.
Depending on the choice of definition, we can estimate the $\AVAR$
either via quantile estimation or density estimation.
In this section, we adopt the first definition and introduce the
following 
$\AVAR$ estimator of \cite{YN13}
\small
\begin{align*}
  \hat G^{A,\lambda}_i(N)  
  \triangleq - \frac{1}{\lambda} \left( \sum_{j=0}^{\lfloor \lambda
      N\rfloor-1} \frac{1}{N} X^i_{(j+1)} + \left(\lambda -
      \frac{\lfloor \lambda N\rfloor}{N} \right) X^i_{(\lceil \lambda
      N \rceil)} \right)
\end{align*}
\normalsize
which is a Riemann sum of $\VAR$ estimators.
Observe that it is computationally more efficient than that of \cite{Brown07}.

\begin{theorem}[$\AVAR$ sample complexity]\label{pro:2}
  Suppose that the assumptions of Theorem~\ref{thm:var} hold.   
  Suppose that the rewards are bounded such that $\abs{X^i_t} \leq M$ almost surely, for every arm $i$ and time $t$.
  In addition, we assume that there exist $D$ and $D'$ such that $d_i(z) \leq D$ and $d'_i(z) \leq D'$ for all $z \in R$ and all $i$, and that
  \small
  \begin{align*}
    N \geq \max\left\{\frac{32 \lambda'(N) M^2}{\eps^2 \lambda^2}\log(2/\delta) , \frac{(1/6)D'/D^3 + 2 C_1 \lambda'}{\eps \lambda} , 2\right\},
  \end{align*}
  \normalsize
where $\lambda'(N)$ denotes the smallest real number greater than $\lambda$ such that $N\lambda'(N)$ is an integer.
  Then, we have, for every arm $i$,
  \begin{align*}
    \P \left( \abs{G^{A,\lambda}_i - \hat G^{A,\lambda}_i(N)} > \eps\right)
    \leq \delta.
  \end{align*}
\end{theorem}

\noindent
This implies the following statement.
\begin{corollary}
Using the notation from Theorem~\ref{pro:2}, suppose that 
  \begin{align*}
    N \geq \max\left\{\frac{(1/6)D'/D^3 + 2 C_1 \lambda'(N)}{\eps \lambda} , 2\right\}.
  \end{align*}
Given this, we have
  \begin{align*}
    \P \left( \abs{G^{A,\lambda}_i - \hat G^{A,\lambda}_i(N)} > \eps\right) \leq 2\exp{\Big\{ -\frac{N\eps^2 \lambda^2}{32 \lambda'(N) M^2} \Big\}}.
  \end{align*}
\end{corollary}

\noindent
From this corollary, we can also derive the following.
\begin{corollary}
Using the notation from Theorem~\ref{pro:2}, suppose that 
  \begin{align*}
    \Big\lfloor \frac{T}{H} \Big\rfloor \geq \max\left\{\frac{(1/6)D'/D^3 + 2 C_1 \lambda'(N)}{\eps \lambda} , 2\right\}.
  \end{align*}
  Given this, the recommendation error for the average value-at-risk functional case has the following upper bound:
\begin{eqnarray*}
e_r(T) \leq 4\Big(H - K +1 \Big)\exp{\Big\{ -\frac{(T-H)\eps^2 \lambda^2}{32 H \lambda'(\lfloor\frac{T}{H} \rfloor) M^2} \Big\}}.
\end{eqnarray*}  
\end{corollary}


\section{Information Theoretic Functionals}


Optimal-arm identification with information theoretic functionals are quite common in natural language processing applications.
For example, it is desirable to find an arm that maximizes the entropy in~\citep{BDD96}.
Given this, this section investigates functional bandit optimisation with entropy and its other counterparts.
To do so, we first describe the required conditions that we focus on within this section.
We then provide refined error probability bounds.

\begin{assumption}[Discrete random variables]
  Suppose that the rewards $X^i_1,\ldots,X^i_N$ are \iid, take values in a
  countable set $V$, and with distribution $F_i$.
\end{assumption}

Let $G^H_i$ denote the entropy of arm $i$:
\begin{align*}
  G^H_i = H(F_i) = \sum_{v\in V} F_i(v) \log_2 F_i(v).
\end{align*}

Let $F_i(N)$ denote the empirical frequency of realisations in the
samples $X^i_1,\ldots,X^i_N$.
A number of consistent estimators exist for the entropy: e.g.,
plug-in, matching-length \citep{AK01}, and nearest-neighbour
estimators.
We consider the $k$-nearest neighbour entropy estimator of \citep{SRH}.

\begin{theorem}{\citep{SRH}}
\label{thm:SRH}
  Suppose that $X^i_1,X^i_2,\ldots$ are $d$-dimensional \iid random
  variables and that the reward distribution of arm $j$ admits a density
  function with bounded support.
  Let $c_1,c_2,c_4,c_5$ denote constants that depend on $F_i$ and $G^H_i$ only. 
  The $k$-nearest neighbor entropy estimator with parameter $M$ satisfies:
 \footnotesize
  \begin{eqnarray}
  \nonumber
    V^i_N &=& \E \abs{ \hat G^H_i(N) - G^H_i } \\
   \nonumber
     &=& c_1 \left(\frac{k}{M}\right)^{1/d}
    + c_2 / k + o(1/k + (k/M)^{1/d}),\\
    \nonumber
    W^i_N &=& \E [ \hat G^H_i(N) - \E \hat G^H_i(N) ]^2 \\
    \nonumber
     &=& c_4/N + c_5/M + o(1/M + 1/N).
  \end{eqnarray}
  \normalsize
\end{theorem}

\begin{theorem}[Entropy estimation error]
Suppose that the assumptions of Theorem~\ref{thm:SRH} hold.
We have
  \begin{align*}
  \P( \abs{ \hat G^H_i(N) - G^H_i } > \eps ) \leq 
  \frac{W^i_N}{(\eps - \abs{V^i_N})^2}.
\end{align*}
\end{theorem}

\noindent
The proof is similar to the proof of Theorem~\ref{thm:var}, and thus, is omitted.
In addition, we have the following result, similar to the case of value-at-risk.
\begin{corollary}
Suppose that the assumptions of Theorem~\ref{thm:SRH} hold. 
Given this, the recommendation error of our algorithm for the entropy case has the following upper bound:
\begin{eqnarray*}
e_r(T) \leq 4\Big(H - K +1 \Big)\frac{W}{(\frac{d}{2} - \abs{V})^2} 
\end{eqnarray*}  
where $W = \max_{i \neq i^*}{W^i_{\lfloor \frac{T}{H}\rfloor}}$ and $V = \max_{i \neq i^*}{V^i_{\lfloor \frac{T}{H}\rfloor}}$.
\end{corollary}

\begin{remark}[Other information theoretic functionals]
  A similar result holds for R\'{e}nyi entropy---a generalization of the
  notion of Shannon entropy, cf. \citep{SRH} and \citep{PPS},
  divergence, and mutual information.
\end{remark}

\section{Conclusions}
\noindent
In this paper we introduced the functional bandit optimisation problem, where the goal is to find an arm with the optimal value of a predefined functional of the arm-reward distributions. 
To tackle this problem, we proposed Batch Elimination, an algorithm that combines efficient functional estimation with arm elimination.
In particular, assuming that there exists a Q-efficient estimator of the functional, we run a number of arm pulling rounds, and eliminate a certain number of weakest remaining arms.
The algorithm stops when there is only one arm left.
We analysed the performance of the algorithm by providing theoretical guarantees on its recommendation error, regret, and PAC regret, respectively.
We also refined our results in a number of cases where we use risk management and information theoretic functionals.

The most trivial way to extend our results to other functionals is via ``plug-in'' functional estimators.
In addition, it is also feasible to modify our algorithm to output a ranking of arms.
Note that when there are multiple arms that are approximately equally good, we may want to extend our analysis to give probabilities of returning the $j$-th best arm.

Recall that our algorithm requires the knowledge of the total number of pulls $T$ in advance. 
However, this is not always the case in many applications.
Given this, it is a desirable goal to extend our algorithm to run in an online fashion, without requiring the time horizon $T$ as an input.

\bibliographystyle{plainnat}
\bibliography{TY14-arxiv}

\newpage
\clearpage
\appendix

\section{Appendix: Proofs}
\label{section:proofs}

\begin{proof}[Proof of Lemma~\ref{lemma:incorrect_prob}]
Since $ \hat G_{i}$ are proper estimators, we have:
\begin{eqnarray}
\nonumber
\P( \hat G_{i}(n) \geq \hat G_{i^*}(n) ) &\leq& \P\Big(\hat G_{i}(n) > G_{i} + \frac{\gamma_{i}}{2}\Big)  \\
\nonumber
&+&  \P\Big(\hat G_{(i^*)}(n) < G_{(i^*)} - \frac{\gamma_{i}}{2}\Big)\\
\nonumber 
&\leq& 2Q \Big(n,\frac{\gamma_j}{2} \Big)
\end{eqnarray} 
\end{proof}
\begin{proof}[Proof of Lemma~\ref{lemma:error_one_round_generic}]
Let $e_p(l)$ denote the probability that the best arm is not eliminated until episode $l$ but will be dropped out at this episode. Consider the $|S_l|-1$ suboptimal arms, and suppose that $i(1), i(2), \dots, i(|S_l| - x_l)$ are the best $|S_l| - x_l$ arms among them. The best arm is eliminated if the estimate of its functional is lower than all of arms $i(1), i(2), \dots, i(|S_l| - x_l)$.
Given this, we have:
\begin{eqnarray}
\nonumber
e_p(l) &\leq& \sum_{j=1}^{|S_l| - x_l}{P(\hat{G}_{i^*}(n_l) \leq \hat{G}_{i(j)}(n_l))} \\
\nonumber
&\leq& 2(|S_l| - x_l)Q\Big(l\frac{T-H}{H},\frac{d}{2}\Big)
\end{eqnarray}
where $n_l$ denotes the total number of pulls per arm $i$ up to round $l$.
Since $n_l = l\Big\lfloor \frac{T}{H} \Big\rfloor \geq l\Big(\frac{T}{H} - 1 \Big)$, we obtain the desired inequality by replacing $n_l$ with $l\frac{T-H}{H}$ (recall the function $Q$ is strictly monotone decreasing in $n$). 
\end{proof}
\begin{proof}[Proof of Theorem~\ref{thm:mv}]
  Recall that $\hat \mu(j)$ and $\hat \sigma^2(j)$ are unbiased
  estimators.  By Hoeffding's inequality,
  we have
  \begin{align*}
    \P(\abs{\hat \mu(j) - \mu(j)} > \gamma_j/4)
    &\leq 2 \exp\left(- \frac{N \gamma_j^2}{8(B-A)^2}\right).
  \end{align*}
  Observe that $(X^i_k - \hat \mu(j))^2 \in [0,(B-A)^2]$ with
  probability 1.  By the triangle inequality, Hoeffding's inequality,
  we have
  \begin{align*}
    &\P\left(\abs{\hat \sigma^2(j) - \sigma^2(j)} > \gamma_j/(4\lambda)\right)\\
    &\leq 2 \exp\left(- \frac{N (\frac{N-1}{N}\gamma_j/\lambda)^2}{8 (B-A)^4}\right).
  \end{align*}
  Observe that
  \begin{align*}
    &\{\abs{-\hat \mu(j) +\lambda \hat \sigma^2(j) + \mu(j) -
      \lambda\sigma^2(j)} > \eps \}\\ &\subseteq \{\abs{\hat \mu(j) -
      \mu(j)} + \abs{\lambda \hat \sigma^2(j) - \lambda\sigma^2(j)} >
    \eps\},
  \end{align*}
  Hence, we have 
  \begin{align}
    &\P\left( \abs{\hat G^{M,\lambda}_{i}(N) - G^{M,\lambda}_{i}} > \gamma_j/2
    \right)\\
    &= \P(\abs{-\hat \mu(j) +\lambda \hat \sigma^2(j) + \mu(j) - \lambda\sigma^2(j)} > \gamma_j/2)\\
    &\leq \P(\abs{\hat \mu(j) - \mu(j)} + \abs{\lambda \hat \sigma^2(j) - \lambda\sigma^2(j)} > \gamma_j/2)\\
    \mbox{(Union bound)}\quad&\leq \P(\abs{\hat \mu(j) - \mu(j)} >
    \gamma_j/4)\\
    &+ \P(\lambda \abs{\hat \sigma^2(j) - \sigma^2(j)} >
    \gamma_j/4),\label{eq:872} 
  \end{align}
where the last inequality follows from a union bound.

Finally, we have
\begin{align*}
&\P( \hat G^M_{i}(N) - \hat G^M_{i^*}(N) \geq 0)\\
&= \P\Big( \hat G^M_{i}(N) - G^M_{i} + G^M_{i} - \\
  &G^M_{i^*} + G^M_{i^*} - \hat G^M_{i^*}(N) \geq 0 \Big)\\
&= \P\Big( \hat G^M_{i}(N) - G^M_{i} + G^M_{i^*} - \hat
G^M_{i^*}(N) \geq \gamma_j \Big)\\
\mbox{(Union bound)}\quad&\leq \P( \hat G^M_{i}(N) - G^M_{i} \geq \gamma_j/2) \\
&+ \P(G^M_{i^*} - \hat
G^M_{i^*}(N) \geq \gamma_j/2)\\
&\leq \P( \abs{ \hat G^M_{i}(N) - G^M_{i} } \geq \gamma_j/2) \\
&+ \P(\abs{G^M_{i^*} - \hat G^M_{i^*}(N)} \geq \gamma_j/2)\\
\mbox{(by \eqref{eq:872})}\quad&\leq 
2 \exp\left(- \frac{N \gamma_j^2}{8(B-A)^2}\right) \\ 
&+ 2 \exp\left(- \frac{N (\frac{N-1}{N}\gamma_j/\lambda)^2}{8 (B-A)^4}\right).
\end{align*}
\end{proof}

\begin{proof}[Proof of Theorem~\ref{thm:var}]
By Theorem~\ref{thm:bahadur}, we have
  \begin{align}\label{eq:334}
    \abs{\E X^i_{(\lceil \lambda N \rceil)} - q_i(\lambda)} = \abs{V^i_N}.
  \end{align}
  By the Triangle Inequality, Equation~\eqref{eq:334}, and Chebyshev's
  Inequality, we have
  \begin{align*}
    &\P\left( \abs{\hat G^{V,\lambda}_i(N) - G^{V,\lambda}_i} \geq \eps \right)\\
    &=\P\left( \abs{X^i_{(\lceil \lambda N \rceil)} - q_i(\lambda)} \geq \eps \right)\\
    &\leq \P\left( \abs{X^i_{(\lceil \lambda N \rceil)} - \E \hat G^{V,\lambda}_i(N)} + \abs{\E X^i_{(\lceil \lambda N \rceil)} - q_i(\lambda)} \geq \eps \right)\\
    &\leq \P\left( \abs{X^i_{(\lceil \lambda N \rceil)} - \E X^i_{(\lceil \lambda N \rceil)}} \geq \eps - \abs{V^i_N} \right)\\
    &\leq \frac{\var X^i_{(\lceil \lambda N \rceil)}}{(\eps - \abs{V^i_N})^2} \leq
    \frac{W^i_N}{(\eps - \abs{V^i_N})^2}.
  \end{align*}
\end{proof}
\begin{proof}[Proof of Lemma~\ref{lemma:VAR_incorrect_prob}]
We have:
\begin{eqnarray*}
\P( \hat G^{V,\lambda}_{i}(N) \geq \hat G^{V,\lambda}_{i^*}(N) ) &\leq& \P( \hat G^{V,\lambda}_{i}(N) \geq  G^{V,\lambda}_{i} + \gamma_{j}/2) \\
&+& \P(\hat G^{V,\lambda}_{i^*}(N) \leq G^{V,\lambda}_{i^*} - \gamma_j/2) \\
&\leq& \P( \hat G^{V,\lambda}_{i}(N) \geq  G^{V,\lambda}_{i} + d/2) \\
&+& \P(\hat G^{V,\lambda}_{i^*}(N) \leq G^{V,\lambda}_{i^*} - d/2) \\
&\leq& \P( |\hat G^{V,\lambda}_{i}(N) - G^{V,\lambda}_{i}| \geq d/2) \\
&+& \P(|\hat G^{V,\lambda}_{i^*}(N) \leq G^{V,\lambda}_{i^*}|  \geq d/2)  \\
&\leq& 2\frac{W^i_N}{(\frac{d}{2} - \abs{V^i_N})^2} 
\end{eqnarray*}
\end{proof}

\begin{proof}[Proof of Theorem~\ref{pro:2}]
  By the definitions of $\rho^A_\lambda$ and $\hat Y^i_\lambda$, and by the Triangle Inequality,
  we have
  \begin{align*}
    &\lambda \abs{\rho^A_\lambda(X^i) -(- \hat Y^i_\lambda)} =\\
    &\abs{ \int_0^\lambda q_i(\xi) \; d\xi
      - \left( \sum_{j=0}^{\lfloor \lambda N\rfloor-1} \frac{X^i_{(j+1)}}{N}
         + \left(\lambda - \frac{\lfloor \lambda
            N\rfloor}{N} \right) X^i_{(\lceil \lambda N \rceil)}
      \right) }\\
    &\leq \abs{ \int_0^{\lambda'} q_i(\xi) \; d\xi
      - \left( \sum_{j=0}^{\lambda' N-1} \frac{1}{N}
        X^i_{(j+1)} \right)}\\
    &\leq \underbrace{\abs{ \int_0^{\lambda'} q_i(\xi) \; d\xi
      - \left( \sum_{j=0}^{\lambda' N-1} \frac{1}{N}
        \E X^i_{(j+1)} \right)}}_{R}\\
    &+ \underbrace{\abs{ \left( \sum_{j=0}^{\lambda' N-1} \frac{1}{N}
        \E X^i_{(j+1)} \right)
      - \left( \sum_{j=0}^{\lambda' N-1} \frac{1}{N}
        X^i_{(j+1)} \right)}}_{S}.
  \end{align*}
  Observe that by Theorem~\ref{thm:bahadur}, we have
  \begin{align*}
    R
    &\leq \abs{ \int_0^{\lambda'} V^i_N(\lambda) \ud \lambda}\\
    &\leq \int_0^{\lambda'} \left( \frac{\lambda (1-\lambda) D'}{2 D^3 (N+2)} + \frac{C_1}{N^2} \right) \ud \lambda\\
    &= \frac{(\lambda'^2/2 - \lambda'^3/3) D'}{2 D^3 (N+2)} + \frac{C_1 \lambda'}{N^2} \triangleq Q.
  \end{align*}
By assumption on $N$, we can verify by simple algebra that $Q < \lambda \eps / 2$.

Observe that 
\begin{align*}
  &\P\left( \abs{\rho^A_\lambda(X^i) -(- \hat Y^i_\lambda)} > \eps\right)\\
  &\leq \P(R + S > \lambda \eps)\\
  &= \P(S > \lambda \eps - R)\\
  &\leq \P(S > \lambda \eps - Q)\\
  &= \P\left(\abs{ \sum_{j=0}^{\lambda' N-1} \frac{1}{N}
        \E X^i_{(j+1)}
      - \frac{1}{N}
        X^i_{(j+1)} } > \lambda \eps - Q \right)\\
  &\leq 2 \exp\left(-\frac{(\lambda \eps - Q)^2 N^2}{2 \lambda' N (2M)^2}\right)
  \leq 2 \exp\left(-\frac{(\lambda \eps/2)^2 N}{8 \lambda' M^2}\right) \leq \delta.
\end{align*}
where the last two inequalities follow by Azuma's Inequality for bounded-difference martingale sequences and the assumption on $N$.
\end{proof}

\end{document}